\documentclass{svproc}
\usepackage{url}

\usepackage[utf8]{inputenc}
\usepackage{graphicx}
\usepackage{amsmath}
\usepackage{amssymb,xcolor}
\usepackage{mathtools}
\usepackage{cancel}
\usepackage{enumitem}
\usepackage[noend]{algorithmic}
\usepackage{algorithm}
\usepackage{wrapfig}
\usepackage{hyperref}
\usepackage{tikz-cd}
\usepackage{tikz}
\usetikzlibrary{automata, positioning}

\usepackage[scriptsize]{subfigure}
\usepackage{epsfig} 

\usepackage{color}

\def\Re{{\mathbb{R}}}

\def\histx{{\tilde{x}}}

\def\histy{{\tilde{y}}}
\def\histY{{\tilde{Y}}}
\def\histu{{\tilde{u}}}
\def\histU{{\tilde{U}}}

\def\E{{\cal{E}}}

\def\is{{\iota}}  
\def\his{{\eta}}  
 
\def\imap{{\kappa}}

\def\ndtrans{{\cal T}}
\def\fmaphist{{\phi_{hist}}}

\def\xdot{{\dot x}}

\def\atan2{\operatorname{atan2}}
\def\pow{{\rm pow}}

\def\ifs{{\cal I}} 
\def\ifsder{{\cal I}_{der}} 
\def\ifshist{{\cal I}_{hist}} 
\def\ifsmin{{\cal I}_{min}}
\def\ifst{{\cal I}_{task}}

\def\imapb{{\kappa_{task}}} 
\def\ifstask{{\cal I}_{task}} 

\newcommand{\cat}{{}^\frown}

\usepackage{acronym}
\acrodef{Ispace}[I-space]{\emph{information space}}
\acrodef{Istate}[I-state]{\emph{information state}}
\acrodef{Imap}[I-map]{\emph{information mapping}}
\acrodef{ITS}[ITS]{\emph{information transition system}}
\acrodef{ITSs}[ITSs]{\emph{information transition systems}}
\acrodef{DITS}[DITS]{\emph{deterministic information transition system}}
\acrodef{NITS}[NITS]{\emph{nondeterministic information transition system}}
\acrodef{POMDPs}[POMDPs]{\emph{partially observable Markov decision processes}}
\acrodef{PSRs}[PSRs]{\emph{predictive state representations}}

\begin{document}
\mainmatter              
\title{The Limits of Learning and Planning: \\ Minimal Sufficient Information Transition Systems}
\titlerunning{Minimal Sufficient Information Transition Systems}  
%
\author{Basak Sakcak, Vadim Weinstein, Steven M. LaValle%
\thanks{This work was supported by a European Research Council Advanced Grant (ERC AdG, ILLUSIVE: Foundations of Perception Engineering, 101020977), Academy of Finland (projects PERCEPT 322637, CHiMP 342556), and Business Finland (project HUMOR 3656/31/2019).  {\tt\small (e-mail: firstname.lastname@oulu.fi).}}%
}
\institute{Center of Ubiquitous Computing \\
Faculty of Information Technology and Electrical Engineering \\
University of Oulu, Finland}

\maketitle              

\begin{abstract}
In this paper, we view a policy or plan as a transition system over a space of information states that reflect a robot's or other observer's perspective based on limited sensing, memory, computation, and actuation.  Regardless of whether policies are obtained by learning algorithms, planning algorithms, or human insight, we want to know the limits of feasibility for given robot hardware and tasks. Toward the quest to find the best policies, we establish in a general setting that minimal information transition systems (ITSs) exist up to reasonable equivalence assumptions, and are unique under some general conditions.
We then apply the theory to generate new insights into several problems, including optimal sensor fusion/filtering, solving basic planning tasks, and finding minimal representations for feasible policies. 
\end{abstract}
%

\section{Introduction}\label{sec:intro}

Robotics increasingly appears as an application area for other fields.
It is a frequent target for designing and testing machine learning
algorithms, planning algorithms, sensor fusion methods, control laws,
and so on.  This may lead many to believe that robotics itself does
not have its own, unique theoretical core (on this, we agree with
Koditschek \cite{Kod21}).  Thus, are we missing something?  Surely
robotics is not a pure algorithmic problem or pure nonlinear
control problem.  Could there be a theory that plays a similar role
to Turing machines for computer science, or $\xdot = f(x,u)$ over
differentiable manifolds for control theory, and yet is distinct from
both?  Can we formulate and potentially answer questions such as: Does a solution even exist to the problem?  What are the minimal necessary components to solve it?  What should the best learning approach imaginable produce as a representation?  Such questions would be analogous to existence and uniqueness in control and dynamical systems, or decidability and complexity (especially Kolomogorov) in theoretical computer science.

\begin{figure}[t]
\begin{center}
\begin{tabular}{ccc}
    \includegraphics[width=0.35\linewidth]{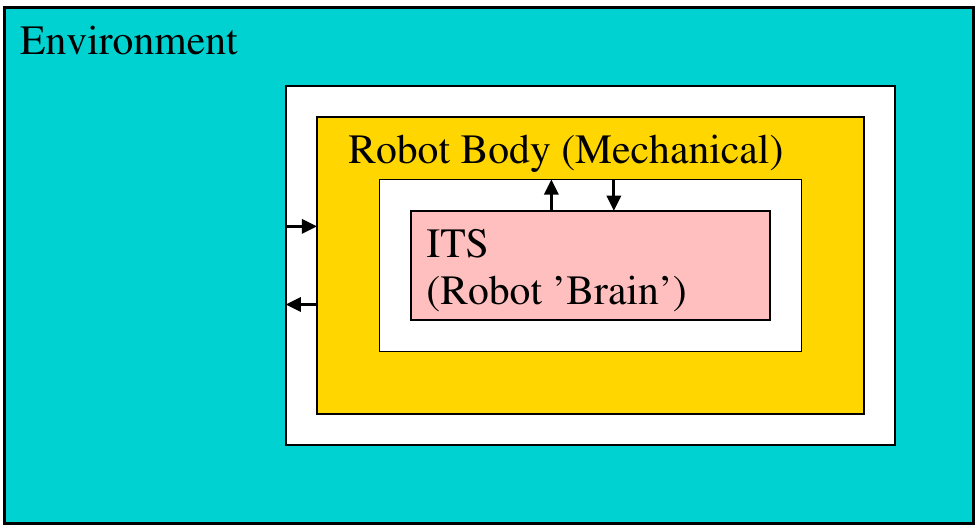} & \hspace*{0.5cm} &
    \includegraphics[width=0.35\linewidth]{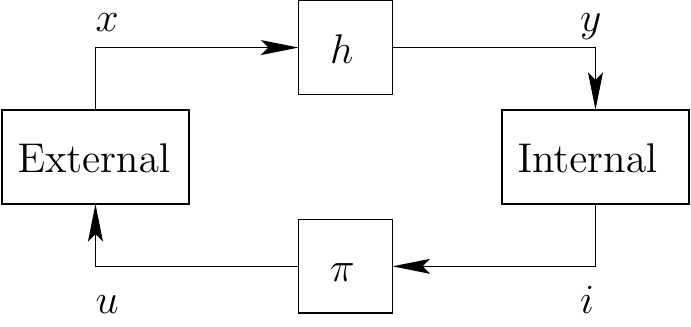} \\
    (a) & & (b)
\end{tabular}
\end{center}\vspace{-2em}
\caption{\label{fig:extint}(a)The {\em internal} robot brain is defined as an ITS that interacts with the {\em external} world (robot body and environment). (b)Coupled internal and external systems mathematically capture sensing, actuation, internal computation, and the external world.}
\vspace{-1.8em}
\end{figure}

This paper proposes a robotics theory that is built from the input-output relationships between a programmable mechanical system (robot) and its
environment via sensing and actuation; see Figure
\ref{fig:extint}(a).  The key is to focus on necessary and sufficient conditions that a robot's internal processor (i.e., ``brain") must maintain to solve required tasks, such as collision-free navigation or coverage. 
We assume 
that the robot hardware and actuation model are fixed, and that for a given set of tasks in a space of environments, we must determine the weakest amount of sensing, actuation, and computation that would
be sufficient for solving tasks.  We will call such conditions {\em minimal sufficient}: If you take away anything from a minimal sufficient system, the tasks will become unsolvable.

We introduce the notion of an {\em information transition system} (ITS) to formally model the robot's brain (as well as any other system observers).  The ``information" part of an ITS is inspired by von Neumann's definitions in the context of
sequential games with hidden information 
(and {\em
  not} Shannon's later notion of information theory). 
This inspired the development of {\em information spaces} \cite{Lav06}
(Chapter 11) as a foundation of planning with imperfect state
information due to sensing uncertainty.  The concept of {\em
  sufficient information mappings} appears therein.  It is generalized
in this paper, and the state space of each ITS will in fact be an
information space.

In our work, the ITS and its underlying information space serve as the domain over which a plan or policy can be expressed and analyzed.  Note that prior work in planning usually assumes that the space one plans over is fixed, as in a configuration space or state (phase) space based on the robot's mobility.  Even the information spaces described in \cite{Lav06} remain fixed in the planning phase.  A notable exception is by O'Kane and Shell \cite{OkaShe17}, in which information spaces for passive filtering are reduced algorithmically, and is closely connected to this paper.  All such spaces will be considered here as potential information spaces, and we intend to reduce or collapse them as much as possible in the development of an information-feedback plan.  This is perhaps closer to the goals of machine learning, in which candidate representations are determined through optimization of discrepancies with respect to input-output data.  In this paper, we in fact consider both model-free and model-based ITSs, in alignment with the choices commonly found in machine learning \cite{BruBerBraLecHasRusGro22}.

The robot's ITS is coupled to the physical world, which is itself modeled as a
transition system.  Note that the physical world model is not ``given"
to the robot; we will formalize notions of ``who gets what" information
in Section \ref{sec:suff}.  The coupled system is inspired by neuroscience
models (for example, \cite{Fri10}).  Many of the concepts in this
paper build upon \cite{WeiSakLav22}, in which we recently proposed an
enactivist-oriented model of cognition based on information spaces.  By {\em enactivist}
\cite{HutMyi12}, it is meant that the necessary brain structures
emerge from sensorimotor interaction and do not necessarily have
predetermined, meaningful components as in the classical
representationalist sense.


Section \ref{sec:math} provides a mathematical formulation of robot-environment interaction as transition systems.  Section \ref{sec:suff} develops notions of sufficiency and minimality over the space of possible ITSs. Section \ref{sec:problems} applies the general concepts to address what it means to solve both passive (filtering) and active 
(planning/control)
tasks minimally and Section \ref{sec:ex} provides simple examples.

\section{Mathematical Models of Robot-Environment Systems}\label{sec:math}

\subsubsection{Internal and external systems}

In this paper, we will consider a robot embedded in an environment and describe this system as two subsystems, named {\em internal} and {\em external}, connected through symmetric input-output relations. External refers to the system describing the physical world, and internal is the complement of it. 
This interaction is shown in Figure~\ref{fig:extint}(b). 
In this sense, the states of the external and internal systems are similar to the use of the term in control theory and computer science, respectively.

External system corresponds to the totality of the environment and the robot body within it. 
Let $X$ denote the set of states of this system; a state could be for example, the configuration of the robot in a known environment (or within a set of possible environments) or its phase.
There are no restrictions on $X$; it may be discrete, an $n$-dimensional manifold, a function space, and so on. 
Next, let $U$ be the set of control inputs (also referred to as actions)
such that when applied at state $x \in X$ causes it to change according to a state transition function $f: X \times U \rightarrow X$. The set $U$ can also be anything: a finite or infinite discrete set, a compact or non-compact manifold, and so on. 
Similarly, at each state $x$, $y=h(x)$ is the output in
which $h: X \rightarrow Y$ is a state-based sensor mapping and $Y$ is the
set of all possible observations. 

The internal system (robot's brain)
observes the external system through a sensor mapping and interacts with it through a selection of 
actions with respect to a policy (alternatively, we can call it a plan or a strategy). Therefore, the input to the internal system is an observation and its output is an action. 
The states of this system correspond to the retained information 
gathered through the outcomes of actions in terms of sensor observations. To this end, the basis of our mathematical formulation of the internal is the notion of \emph{\ac{Ispace}} \cite{Lav06}. We will use the term \emph{\ac{Istate}} to refer to the state of the internal system and denote it with $\is$, and $\ifs$ will denote the set of all I-states, that is, the \emph{\ac{Ispace}}. Note that the notions \ac{Ispace} and \ac{Istate} are not exclusive to the internal system and we will use them in a more general setting in the following sections.
Similar to the external system, the internal system evolves with each $y \in Y$ according to the information transition function $\phi: \ifs \times Y \rightarrow \ifs$. The output then corresponds to the control command given by an information feedback policy $\pi: \ifs \rightarrow U$. 

Finally, we can write the coupled dynamical system composed of these two subsystems defined as external and internal as 
\begin{align}
\centering
    x'&=f(x, \pi(\is)) & y=&h(x) \nonumber \\
    \is'&=\phi(\is,h(x)) & u=&\pi(\is). \label{eqn:coupled_sys}
\end{align}
Whereas the equations on the left side describe the evolution of this coupled system, the ones on the right show the respective outputs of each subsystem. 
Given an initial state $(x_1, \is_1) \in X \times \ifs$, there exists a unique state-trajectory. 

Suppose the system evolves in discrete stages. 
For the external system, starting from an initial state $x_1$, each stage $k$ corresponds to applying an action $u_{k}$ which then yields the next stage $k+1$ and the next state $x_{k+1}=f(x_{k},u_{k})$. 
As the system evolves through stages, $\histx_k=(x_1,x_2,\dots,x_k)$, $\histu_{k-1}=(u_1,u_2,\dots,u_{k-1})$, $\histy_k=(y_1,y_2,\dots, y_k)$, correspond to the state, action and observation histories up to stage $k$, respectively. Note that applying the action $u_k$ at stage $k$ would result in a transition to state $x_{k+1}$ and the corresponding sensor reading $y_{k+1}=h(x_{k+1})$. 
The same applies for the internal system, we can describe its evolution, starting from an initial \ac{Istate} $\is_0$ following the state transition equation $\is_{k}=\phi(\is_{k-1},y_k)$. At stage $k$, $\pi(\is_k)$ would produce the action $u_k$. Note that the stage index of the \ac{Istate} starts from $0$, this corresponds to any prior information the internal system might have regarding the external; $\is_1$ is then obtained using $\is_0$ and $y_1$. Furthermore, $\pi(\is_0)=u_0=()$ for all policies $\pi$, meaning that no 
action is outputted at this stage.

\subsubsection{Generalizing to transition systems}

Without loss of generality, we can describe the internal and external subsystems as transition systems of the form $(S, \Lambda, T)$ in which $S$ is the set of states, $\Lambda$ is the set of names for the outgoing transitions, and $T \subset S \times \Lambda \times S$ is a ternary relation describing the transitions. If for each $(s,\lambda) \in S \times \Lambda$ there is a unique $s' \in S$ such that $(s,\lambda,s') \in T$, then we will write this system as $(S, \Lambda, \tau)$ in which $\tau: S\times \Lambda \rightarrow S$ is a function, and call the system an {\em automaton}. 
This corresponds to a deterministic system. Note that our definition of an automaton differs from the one usually used in computer science in the sense that ours do not necessarily have a start state and a set of accepting states, and it is not necessarily finite. Suppose $\ndtrans : S \times \Lambda \rightarrow \pow(S)$, in which $\pow(\cdot)$ denotes the power set. Then, the transition system $(S, \Lambda, \ndtrans)$ is a nondeterministic automaton. 

In \cite{WeiSakLav22}, we have 
used the notion of state-relabeled transition systems to model the internal and external systems. A state-relabeled transition system
is the quintuple $(S,\Lambda,T,\sigma,L)$ in which $\sigma: S \rightarrow L$ is a labeling function and $(S,\Lambda,T)$ is a transition system. Preimages of a labeling function $\sigma$ induce a partitioning of the state space $S$. Let $S/\sigma$ be the set of equivalence classes $[s]_\sigma $ induced by $\sigma$ such that $S/\sigma=\{[s]_\sigma \mid s\in S\}$ and $[s]_\sigma=\{s'\in S \mid \sigma(s')=\sigma(s)\}$. Then, we can define a new transition system $(S/\sigma,\Lambda,T/\sigma)$ called the \emph{quotient} of $(S,\Lambda,T)$ by $\sigma$, in which $T/\sigma=\{\left([s]_\sigma, \lambda, [s']_\sigma\right) \mid (s, \lambda, s') \in T\}$. Note that $(S/\sigma,\Lambda,T/\sigma)$ is a reduced version of $(S,\Lambda,T)$. 
We might be interested in finding a labeling function $\sigma$ such that the corresponding quotient transition system is as simple as possible while ensuring that it is still useful. In the following sections, we will provide motivations for a reduction and discuss in more detail the requirements on $\sigma$ for the quotient system to be useful. 

Considering the deterministic case and the description of an automaton given above, external and internal systems can be written as state-relabeled automata $(X,U,f,h,Y)$ and $(\ifs,Y,\phi,\pi,U)$, respectively, in which $h$ and $\pi$ are considered as labeling functions.
Interpreting the labels as the output of a transition system,
coupled internal-external system can be described in terms of the state-relabeled transition systems formulation too such that output of one transition system is an input for another. Described this way, coupling of two transition systems result in unique paths in either automaton initialized at a particular state.  

\section{Sufficient Information Transition Systems}\label{sec:suff}

\subsection{Information transition systems}\label{sec:ITS_prespectives}

In the general setting, an \ac{Istate} corresponds to the available (stored) information at a certain stage with respect to the action and observation histories. Consequently, an \ac{Ispace} refers to the collection of all possible \ac{Istate}s. We will use the term \emph{\ac{ITS}} to refer to a transition system whose state space is an \ac{Ispace}.
We have already used the notion of \ac{Ispace} while modeling the internal system representing the robot brain, which makes it an \ac{ITS}. Here, we extend the notion of an \ac{ITS} to include different perspectives from which the external and the coupled system is viewed. In particular, we identified three perspectives corresponding to 1) a plan executor 
which corresponds to the robot brain 2) a planner, and 3) an (independent) observer. With a slight abuse of previously introduced notation and terminology, we will use the term ``internal" to refer to any system that is not the external and we will use $\ifs$ to denote a generic \ac{Ispace}. 
We describe an \ac{ITS} in a robot-centric way such that an observation will refer to a sensor-reading, that is, $y$. However, an independent observer defined over the coupled system
can observe, at stage $k$, both the action taken $u_{k-1}$ and the corresponding sensor-reading $y_k$.

Recall that the information regarding the external is obtained through the sensor-mapping and any potential prior knowledge. Suppose that no policy is fixed over the \ac{Ispace}. Then, the corresponding internal system can be modeled as an \ac{ITS} of the form $(\ifs, U \times Y, \phi$), in which $\phi: \ifs \times (U\times Y) \rightarrow \ifs$ is a state (information) transition function, if it is deterministic. We will then use the term \ac{DITS} to refer to them. Otherwise, it is called a \ac{NITS} and described as $(\ifs, U \times Y, \Phi)$, in which $\Phi \subseteq \ifs \times (U\times Y) \times \ifs$ is the transition relation.
This formulation corresponds to the perspectives other than the plan executor such that it is possible to take any action from an \ac{Istate} as it is not constrained by a policy, hence the outgoing transitions are determined by the elements of $U\times Y$.

A plan executor corresponds to the internal system (robot's brain) described in the previous section. 
The only information regarding the external is gained through manipulating the state of the external system through 
actions
and obtaining the corresponding sensor readings. Recall the representation used in the previous section, that is, $(\ifs, Y, \phi)$, and $\pi:\ifs \rightarrow U$ a labeling function. This can be considered as a constrained version of the \ac{DITS} described in the previous paragraph such that the transitions are restricted to those that can be realized under $\pi$. To show that, we augment the definition of internal system corresponding to the robot brain such that the transitions now also correspond to labels. Let $(\ifs, U \times Y, \Phi)$ be the augmented transition system describing the internal such that 
\begin{equation}\label{eq:augItS}
\Phi= \{ (\is, (u,y), \is') \in \ifs \times (U \times Y) \times \ifs \mid u=\pi(\is) \land \is'=\phi(\is,y)\},
\end{equation}
by construction, this augmented \ac{ITS} is also deterministic\footnote{We could use the same approach for the external system too. In that case, let $(X, Y \times U, F)$ be this augmented transition system corresponding to the external, in which $F$ is the set of transitions such that $F =\{ (x, (y,u), x') \in X \times (Y \times U) \times X \mid y=h(x) \land x'=f(x,u)\}$. Further creating bipartite graphs (for either system) such that transitions from a state correspond either to an observation $y\in Y$ or to an action $u\in U$ allows us to describe the coupling as a form of intersecting two automata. However, because it is not central to this paper we will not elaborate on this topic.}. Suppose $(\ifs, U\times Y, \Phi')$ is the \ac{DITS} that is not constrained by a policy.
Then, $\Phi \subseteq \Phi'$.

\subsection{History information spaces}\label{sec:hist_ispace}
An \ac{Ispace} constitutes the state space of an \ac{ITS}. Therefore, we describe the basic \ac{Ispace} named \emph{history \ac{Ispace}} denoted as $\ifshist$. 
It will be used to derive other \ac{Ispace}s as well. 
A \emph{history \ac{Istate}} at stage $k$ corresponds to all the information that is gathered through sensing (and potentially also through actions) 
up to stage $k$ assuming perfect memory. Let $\his_k$ denote the history \ac{Istate} at stage $k$, that is $\his_k = (\his_0, \histu_{k-1}, \histy_k)$, in which $\his_0$ is the initial condition. 
Recall that $\histu_0$ is assumed to be the null-tuple, hence, $\histu_k$ starts with $u_1$ for any $k>1$.

Let $\ifs_0$ be the set of initial conditions whose description varies with the available prior information. We defer the descriptions of possible $\ifs_0$ to the following paragraph. 
The history information space at stage $k$ is expressed as $\ifs_k=\ifs_0 \times \histU_{k-1} \times \histY_k$. In general, the number of stages that the system will go through is not fixed. Therefore, we can define history \ac{Ispace} as the union over all $k \in \mathbb{N}$, that is, $\ifshist= \bigcup_{k \in \mathbb{N}}\ifs_k$. The \ac{DITS} corresponding to $\ifshist$ becomes $(\ifshist, U \times Y, \fmaphist)$, in which $$\his_{k}=\fmaphist(\his_{k-1},u_{k-1},y_{k})=\his_{k-1}\cat u_{k-1}\cat y_k$$ 
and $\cat$ is the concatenation operation that adds an element at the end of a sequence. 

We consider two categories of initial conditions depending on whether information regarding the state space $X$ of the external system is available or not. Suppose $X$ or any information regarding $X$ is not given. Then, an \ac{Istate} at stage $k$ simply is $\his_k = (\histu_{k-1}, \histy_k)$, that is, the concatenation of action and observation histories up till stage $k$. 
We call this type of history \ac{Ispace}, the \emph{model-free history \ac{Ispace}}, and respectively call the corresponding \ac{ITS}, \emph{model-free history \ac{ITS}}. In this case, we can treat $\his_0$ as $\his_0=()$.
Thus, $\ifs_0=\{()\}$.
For the second category of initial conditions, full or partial information regarding $X$, against which the actions and observations can be interpreted, is given. We will then use the terms \emph{model-based history \ac{Ispace}} and \emph{model-based history \ac{ITS}} to refer to the respective \ac{Ispace} and \ac{ITS}.
The initial condition $\his_0$ could be (i) a known state $x_1 \in X$ such that $\ifs_0 = X$, (ii) a set of possible initial states $X_1 \subset X$ such that $\ifs_0=\pow(X)$ or (iii) a probability distribution $P(x_1)$ over $X$
such that $\ifs_0 \subseteq \mathcal{P}(X)$, in which $\mathcal{P}(X)$ is the set of all probability distributions over $X$. 

\vspace{-0.5em}
\subsection{Sufficient state-relabeling}
In \cite{WeiSakLav22} we have introduced a notion of \emph{sufficiency} that substantially generalizes the definition 
in Chapter 11 of \cite{Lav06} and is presented 
here for completeness.
\begin{definition}[Sufficient state-relabeling]
Let $(S,\Lambda,T)$ be a transition system. A labeling function $\sigma:S \rightarrow L$ defined over the states of a transition system is sufficient if and only if for all $s,q,s',q' \in S$ and all $\lambda \in \Lambda$, the following implication holds:
\[
\sigma(s)=\sigma(q) \land (s,\lambda,s')\in T \land (q,\lambda,q')\in T \implies \sigma(s')=\sigma(q'). 
\]
If $\sigma$ is defined over the states of an automaton $(S,\Lambda,\tau)$, then $\sigma$ is sufficient iff for all $s,q\in S$ and all $\lambda \in \Lambda$, $\sigma(s)=\sigma(q)$ implies that $\sigma(\tau(s,\lambda))=\sigma(\tau(q,\lambda))$.
\label{def:sufficiency}
\end{definition}

Consider the stage-based evolution of
external system $(X,U,f,h,Y)$ with respect to the action (control input) sequence $\histu_{k-1}=(u_1,\dots,u_{k-1})$. This corresponds to the state and observation histories till stage $k$, that are $\histx_k=(x_1,\dots,x_k)$ and $\histy_k=(y_1,\dots,y_k)$. Recall that applying $u_k$ at stage $k$ would result in a transition to $x_{k+1}$ and the corresponding observation $y_{k+1}=h(x_{k+1})$.
Hence, in this context, sufficiency of $h$ implies that given the label $y_k=h(x_{k})$ and the action $u_k$, it is possible to determine the label $y_{k+1}=h(x_{k+1})$. One interpretation of sufficiency of $h$ is that the respective quotient system 
sufficiently represents the underlying system up to the induced equivalence classes. This notion is similar to minimal realization of a system, that is, the minimal state space description that models the given input-output measurements (see for example \cite{KotMooTon18}). Second interpretation is in a predictive sense. 
Suppose the quotient system is known. Then, the label $y_{k+1}=h(x_{k+1})$ can be determined before the system gets to $x_{k+1}$, using the current label $y_k$ and the action to be applied $u_k$. 
Furthermore, under a fixed policy, complete observation-trajectory can be determined from the initial observation by induction. 

Now, consider an internal system with a labeling function $\imap: \ifs \rightarrow \ifs'$, that is, $(\ifs,U\times Y,\phi,\kappa,\ifs')$, and its evolution with respect to the observation history $\tilde{y}=(y_1,\dots,y_{k})$. At stage $k$, the state of the automaton is $\is_{k}$ and with $(u_k,y_{k+1})$ the system transitions to $\is_{k+1} = \phi(\is_k, u_k, y_{k+1})$. Sufficiency of $\imap$ implies that given $\imap(\is_k)$, $u_k$, and $y_{k+1}$, we can determine $\imap(\is_{k+1})$. This is equivalent to the definition introduced in Chapter 11 of \cite{Lav06} and makes it a special case for Definition~\ref{def:sufficiency}. 

\subsection{Derived information transition systems}\label{sec:derived_ITS} 
Even though it seems natural to rely on a history \ac{ITS}, dimension of a history \ac{Ispace} increases linearly with each stage, making it impractical in most cases. 
Thus, we are interested in defining a reduced \ac{ITS} that is more manageable.
Furthermore, this would largely simplify the description of a policy for a planner or a plan executor. 

Recall the quotient of a transition system by a a labeling function. 
We rewrite $(\ifshist, U \times Y, \phi_{hist})$ as $(\ifshist, U \times Y, \Phi_{hist})$, in which 
\begin{equation} \label{eq:Set_Phi}
\Phi_{hist}=\{(\his, (u,y), \phi_{hist}(\his,u,y)) \in \ifshist \times (U \times Y) \times \ifshist \}.     
\end{equation}
We can introduce an \ac{Imap} $\imap: \ifshist \rightarrow \ifsder$ that categorizes the states of $\ifshist$ into equivalence classes through its preimages. 
In this case, $\imap$ serves as a labeling function and a reduction can be obtained in terms of the quotient of $(\ifshist, U \times Y, \Phi)$ by $\imap$, that is, $(\ifshist/\imap, U \times Y, \Phi/\imap)$. 

It is crucial that the derived \ac{ITS} is a \ac{DITS} so that the labels can be determined using only the derived \ac{ITS} without making reference to the history \ac{ITS}.
Considering the quotient system derived from $(\ifshist, U \times Y, \phi)$, which is a \ac{DITS} by definition, by $\imap$, we can not always guarantee that the resulting \ac{ITS} is deterministic. This depends on the \ac{Imap} used for state-relabeling as stated in the following proposition. 

\begin{proposition} \label{prop:non_sufficient_label}
For all non-empty $U$ and $Y$, and for the corresponding $\ifshist$, 
there exists a labeling function $\imap$ such that the quotient of 
$(\ifshist, U \times Y, \phi)$ by $\kappa$, that is, $(\ifshist/\imap, U \times Y,  \Phi/\imap)$, in which $\Phi$ is defined as in \eqref{eq:Set_Phi},
is not a \ac{DITS}. 
\end{proposition}
\begin{proof}
Let $\imap: \ifshist \rightarrow \{l_1, l_2\}$ such that $\imap^{-1}(l_1) = \{\his_{k}=(\histu_{k-1}, \histy_k) \in \ifshist \mid \histu_{k-1}=(u_i)_{i=1,\dots,k-1}, u_i=u, \forall i=1,\dots,k-1 \}$ is the set of histories that correspond to applying the same action for $k-1$ times and $\imap^{-1}(l_2)$ is its complement, that is, $\imap^{-1}(l_2)=\ifshist \setminus \imap^{-1}(l_1)$. Then, there exist sequences $\his_{k-2}=(\histu_{k-3},\histy_{k-2})$ and $\his_{k-1}=(\histu_{k-2},\histy_{k-1})$ such that $\his_{k-2}=\his_{k-1}\cat (u,y)$ and $\his_{k}=\his_{k-1}\cat (u,y)$ for which $\imap(\his_{k-2})=\imap(\his_{k-1})=l_2$ and $\imap(\his_k)=l_1$. Thus, 
\[
\{([\his_{k-2}]_\imap, (u,y), [\his_{k-1}]_\imap), ([\his_{k-1}]_\imap, (u,y), [\his_{k}]_\imap) \} \in \Phi/\imap.
\]
Since $[\his_{k-2}]_\imap=[\his_{k-1}]_\imap$ and $[\his_{k-1}]_\imap \neq [\his_{k}]_\imap$, the transition corresponding to $([\his_{k-1}]_\imap, (u,y))$ is not unique; thus, $(\ifshist/\imap, U \times Y,  \Phi/\imap)$ is not deterministic.\qed
\end{proof}

Note that Proposition~\ref{prop:non_sufficient_label} holds also in the case of a generic \ac{ITS} $(\ifs, U \times Y, \phi)$, with non-history \ac{Istate}s, if $\exists s,s',q,q' \in \ifs$ such that $\{(s,(u,y),s'), (q,(u,y),q')\} \in \Phi$, in which $\Phi$ is defined using $\phi$ as in \eqref{eq:Set_Phi}. Then, any \ac{Imap} $\imap$ such that $\imap(s) = \imap(q)$ and $\imap(s') \neq \imap(q')$ results in a quotient system that is not a \ac{DITS}. 

For the quotient system derived from $(\ifshist, U \times Y, \phi)$ to be a \ac{DITS} depends on the sufficiency of $\imap$. In \cite{WeiSakLav22} 
it is shown that the quotient of a transition system $(S,\Lambda,T)$ by a labeling function $\sigma$ is an automaton (recall our definition) if and only if $(S,\Lambda,T)$ is full\footnote{A transition system $(S,\Lambda,T)$ is full, if $\forall s\in S, \lambda \in \Lambda$ there exists at least one $s'\in S$ with $(s,\lambda,s')\in T$.} and $\sigma$ is sufficient.
As $\phi_{hist}$ is a function with domain $\ifshist \times (U\times Y)$, it is full, then, the following follows from \cite{WeiSakLav22} as a special case. 

\begin{proposition}\label{prop:suff_label}
Let $(\ifshist/\imap, U \times Y,  \Phi_{hist}/\imap)$ be the quotient of $(\ifshist, U \times Y, \phi_{hist})$ by $\imap$, in which $\Phi$ is defined as in \eqref{eq:Set_Phi},
then $(\ifshist/\imap, U \times Y,  \Phi/\imap)$ is a \ac{DITS} if and only if $\imap$ is sufficient.
\end{proposition}

For an \ac{Imap} $\imap: \ifshist \rightarrow \ifsder$, $(\ifshist/\imap, U \times Y, \Phi_{hist}/\imap)$ is isomorphic to $(\ifsder,U \times Y,\Phi_{der})$, in which $\Phi_{der}=\{(\imap(\his),(u,y),\imap(\his')) \mid (\his, (u,y), \his') \in \Phi_{hist}\}$ 
\cite{WeiSakLav22}. 
Thus, we can use the labels introduced by $\imap$ as the new (derived) \ac{Ispace} and the corresponding quotient system as the derived \ac{ITS}. Suppose, $\imap$ is sufficient. Then, the derived \ac{ITS} is a \ac{DITS}, meaning that given the \ac{Istate} $\is_{k-1} \in \ifsder$, and ($u_{k-1}$, $y_k$), $\is_{k+1}\in \ifsder$ can be uniquely determined. Consequently, we can write the derived \ac{ITS} as $(\ifsder,U \times Y,\phi_{der})$ in which $\phi_{der}: \ifsder \times (U \times Y) \rightarrow \ifsder$ is the new information transition function. Therefore, we no longer need to rely on the full histories and the history \ac{ITS} and can rely solely on the derived \ac{ITS}. This is shown in the first two rows of the following diagram: 
\vspace{-1em}
\begin{equation}\label{eqn:dia}
  \begin{tikzcd}[row sep=0.4cm]
    \ifshist \arrow{r}{u_1,y_2} \arrow{d}{\imap} &
    \ifshist \arrow{r}{u_2,y_3} \arrow{d}{\imap} &
    \ifshist \arrow{r}{u_3,y_4} \arrow{d}{\imap} &
    \ifshist \arrow{r}{u_4,y_5} \arrow{d}{\imap} &
    \ifshist \arrow{d}{\imap} \arrow[dash, dotted]{r} & \phantom{.}\\
    \ifsder \arrow{r}{u_1,y_2} \arrow{d}{\imap'} &
    \ifsder \arrow{r}{u_2,y_3} \arrow{d}{\imap'} &
    \ifsder \arrow{r}{u_3,y_4} \arrow{d}{\imap'} &
    \ifsder \arrow{r}{u_4,y_5} \arrow{d}{\imap'} &
    \ifsder \arrow{d}{\imap'} \arrow[dash, dotted]{r} & \phantom{.}\\
    \ifsmin \arrow{r}{u_1,y_2} \arrow{d}{\imap''} &
    \ifsmin \arrow{r}{u_2,y_3} \arrow{d}{\imap''} &
    \ifsmin \arrow{r}{u_3,y_4} \arrow{d}{\imap''} & 
    \ifsmin \arrow{r}{u_4,y_5} \arrow{d}{\imap''} & 
    \ifsmin \arrow{d}{\imap''} \arrow[dash, dotted]{r} & \phantom{.}\\
    \ifst  &
    \ifst  &
    \ifst  &
    \ifst  & \ifst .
  \end{tikzcd}
\end{equation}
Note that we can similarly define an \ac{Imap} that maps any derived \ac{Ispace} to another. An example is given in \eqref{eqn:dia} as the mappings $\imap':\ifsder \rightarrow \ifsmin$ and $\imap'':\ifsmin \rightarrow \ifst$. The corresponding quotient system is deterministic for $\imap'$, indicating that it is sufficient. However, the quotient system by $\imap''$ derived from $\ifsmin$ is not deterministic, hence, $\imap''$ is not sufficient, as the next \ac{Istate} can not be uniquely determined. Note that an \ac{Imap} whose domain is $\ifshist$ can also be defined as composition of the mappings along the column of the diagram. For instance, $\imap_{min}:\ifshist \rightarrow \ifsmin$ is the composition of $\imap$ and $\imap'$, that is, $\imap_{min}=\imap' \circ \imap$ (same for $\imapb: \ifshist \rightarrow \ifst$). 

\subsection{Lattice of information transition systems}

We fix $\ifshist$, which corresponds to fixing the set of initial states $\ifs_0$. Then, each \ac{Imap} $\imap$ defined over $\ifshist$ induces a partition of $\ifshist$ through its preimages, denoted as $\ifshist / \imap$. An \ac{Imap} $\imap'$ is a refinement of $\imap$, denoted as $\imap' \succeq \imap$, if $\forall A \in \ifshist / \imap'$ there exists a $B \in \ifshist / \imap$ such that $A \subseteq B$. Let $K(\ifshist)$ denote the set of all partitions over $\ifshist$. Refinement induces a partial ordering since not all partitions of $\ifshist$ are comparable. 
The partial ordering given by refinements form a lattice of partitions over $\ifshist$, denoted as $(K(\ifshist),\succeq)$. 

At the top of the lattice, there is the partition induced by an identity \ac{Imap} (or equivalently, by a bijection), $\imap_{id}: \ifshist \rightarrow \ifshist$, since all of its elements are singletons (all equivalence classes contain exactly one element), making it the maximally distinguishable case. Conversely, we can define a constant mapping $\imap_{const}: \ifshist \rightarrow \ifs_{const}$ for which $\ifshist/\imap_{const}$ is a singleton, that is, $\ifs_{const}=\{\is_{const}\}$, which then will be at the bottom of the lattice. In turn, $\imap_{const}$ yields the minimally distinguishable case as all histories now belong to a single equivalence class.
This idea is similar to the notion of the \emph{sensor lattice} defined over the partitions of $X$ \cite{Lav12b,ZhaShe21}. Indeed, if we take $\ifs_0=X$ and consider $\imap_{est}: \ifshist \rightarrow X$, the ordering of partitions of $\ifshist$ such that $\ifshist/\imap_{est}$ is the least upper bound gives out the sensor lattice.

As motivated in previous sections, we are interested in finding a sufficient \ac{Imap} such that the quotient \ac{ITS} derived from the history \ac{ITS} is still deterministic. Notice that the constant \ac{Imap} $\imap_{const}$ is sufficient by definition since 
for all $(u,y) \in U\times Y$, and all $\his,\his' \in \ifshist$, we have that $\imap_{const}(\his)=\imap_{const}(\his')$ and $\imap_{const}(\phi_{hist}(\his,(u,y)) = \imap_{const}(\phi_{hist}(\his',(u,y)).$
On the other hand, in certain cases it is crucial to differentiate certain histories from the others. This will become clear in the next section when we describe the notion of a task.  
Suppose $\imap$ is a labeling that partitions $\ifshist$ into equivalence classes that are of importance and suppose that $\imap$ is not sufficient. Then, we want to find a refinement of $\imap$ that is sufficient. 
This will serve as a lower bound on the lattice of partitions over $\ifshist$ since for any partition such that $\ifshist/\imap$ is a refinement of it, the  
classes of histories that are deemed crucial will not be distinguished. The following defines the refinement of $\imap$ that ensures sufficiency and a minimal number of equivalence classes.

\begin{definition}
Let $(\ifshist,U \times Y, \fmaphist)$ be a history \ac{ITS} and $\imap$ an \ac{Imap}. A \emph{minimal sufficient refinement} of $\imap$ is a sufficient \ac{Imap} $\imap'$ such that there does not exist a sufficient \ac{Imap} $\imap''$ that satisfy $\imap' \succ \imap'' \succeq \imap$.
\end{definition}

\begin{remark}
It is shown in \cite{WeiSakLav22} that the minimal sufficient refinement of $\imap$ defined over the states of an automaton $(S,\Lambda,\tau)$ is unique.
\end{remark}

\section{Solving Tasks Minimally}\label{sec:problems}
\subsubsection{Definition of a task}
We now connect the general ITS concepts to the accomplishment of particular tasks.  We have two categories: 1) {\em active}, which corresponds to planning and executing an information-feedback policy that forces a desirable outcome in the environment, and 2) {\em passive}, which means only to observe the environment without being able to effect changes.  Recall from Section \ref{sec:derived_ITS} that there may be model-free or model-based formulations.
In the model-free case, tasks are specified using a logical language
over $\ifshist$ which will result in a labeling and derived I-space
$\ifst$ and associated I-map $\imapb$ that corresponds to the
``resolution'' at which the tasks are specified.  Various logics are
allowable, such as propositional or a temporal logic.  The resulting
sentences of the language involve combinations of predicates that may
assign true or false values to subsets of $\ifshist$.  Solving an active task
(or tasks) requires that a sentence of interest becomes true during
execution of the policy.  This is called {\em satisfiability}.  For
example, the task may be to simply reach some goal set $G \subset
\ifshist$, causing a predicate {\sc in-goal}$(\ifshist)$ to become
satisfied (in other words, be true).  Using linear temporal logic,
more complex requirements, such as cycling through a finite sequence
of subsets forever while avoiding others, can be specified
\cite{FaiGirKrePap09}.  

Solving a passive task only requires maintaining whether a sentence is satisfied, rather than forcing an outcome; this corresponds to filtering.  Whether the task is active or passive, if satisfiability is concerned with a single, fixed sentence, then a {\em task-induced labeling} (or {\em task labeling} for short), that is, $\imapb$, over
$\ifshist$ assigns two labels: Those I-states that result in true and
those that result in false.  A task labeling may also be assigned for
a set of possible sentences by assigning a label to each set of the
common refinement of the partition of $\ifshist$ induced by each
possible sentence.  In the model-based case, tasks are instead
specified using a language over $X$, and sentence satisfiability must
be determined by an I-map that converts history I-states into
expressions over $X$.

\subsubsection{Problem families}
It is assumed that the state-relabeled transition system $(X, U, f, h,Y)$ describing the external system is fixed, but it is unknown or partially known to the observer (a robot or other observer).


Filtering (passive case) requires maintaining the label of an \ac{Istate} attributed by $\imapb$. Since $\imapb$ is not necessarily sufficient, we can not guarantee that the quotient system by $\imapb$ is a \ac{DITS} (Propositions~\ref{prop:non_sufficient_label} and \ref{prop:suff_label}). Thus, relying solely on the quotient system by $\imapb$, we can not determine the class that the current history belongs to (see the last row in \eqref{eqn:dia}) and, hence can not determine whether a sentence describing the task is satisfied (or which sentences are satisfied).

Suppose the sets $U$ and $Y$ are specified, and at each stage $k$, $u_{k-1}$ is known and $y_k$ is observed. The following describes the problem for a passive task given a state-relabeled (history) \ac{ITS} $(\ifshist, U \times Y, \phi_{hist}, \imapb, \ifstask)$, in which $\imapb:\ifshist \rightarrow \ifstask$ is a task labeling that is not sufficient, and $\ifstask$ is the corresponding \ac{Ispace}. 


\begin{problem}[Find a sufficient \ac{Ispace} filter]\label{prob:suff_ref}
Find a sufficient refinement of $\imapb$.
\end{problem}

Note that $\ifshist/\imapb$ determines a lower bound on the partitioning of $\ifshist$ which is interpreted as the crucial information that can not be lost. Consequently, histories belonging to different equivalence classes with respect to $\imapb$ must always be distinguished from each other. 
However, Problem~\ref{prob:suff_ref} does not 
an upper bound. At the limit, a bijection from $\ifshist$ is always a sufficient refinement of $\imapb$. As stated previously, using history \ac{ITS} can create computational obstructions in solving problems. 
This motivates the following problem.



\begin{problem}[Filter minimization]\label{prob:min_suff_ref}
Find a minimal sufficient refinement of $\imapb$.
\end{problem}

We now consider a basic planning problem, for which $\ifst=\{0,1\}$, such that $\imapb^{-1}(1)\subset \ifshist$ is the set of histories that achieve the goal, and $\imapb^{-1}(0)\subset \ifshist$ is its complement. 
Most planning problems refer to finding a labeling function $\pi$ such that, when used to label the states of the internal system, guarantees task accomplishment. Then $\pi$ is called a {\em feasible} policy, which is defined in the following. 
Consider an external system $(X,f,U, h, Y)$.
Let $\mathcal{R}_X(\ifstask) \subseteq X$ be the set of initial states for which there exist a $k$ and histories $\histx_k$, $\histu_{k-1}$, and $\histy_k$, such that $x_{i+1}=f(x_i,u_i)$ and $y_i=h(x_i)$ for all $0<i<k$, and $\his_k \in \imapb^{-1}(1)$, in which $\his_k$ is the history \ac{Istate} corresponding to $\histu_{k-1}$ and $\histy_k$. 

\begin{definition}[Feasible policy for $\ifst$] Let $(\ifs, Y, \phi, \pi, U)$ and $(X,f,U, h, Y)$ be the state-relabeled transition systems corresponding to internal and external systems, respectively. Labeling function $\pi$ defined over $\ifs$ is a feasible policy for $\ifst$ if for all $x \in \mathcal{R}_X(\ifst)$, 
the history corresponding to the coupled internal-external system initialized at $(\is_0, x)$ belongs to $\imapb^{-1}(1)$.
\end{definition}

Most problems in the planning literature consider a fixed \ac{DITS} and look for a feasible policy for $\ifst$. This yields the following problem. Typically, the \ac{Ispace} considered is $X$ which corresponds to state estimation. Note that a \ac{DITS}, in other words, the robot brain, can be seen as an \ac{Ispace} filter itself. 

\begin{problem}[Find a feasible policy]
\label{prob:planning_wDITS}
Given $(\ifs, Y, \phi)$, an internal system (robot brain),
find a labeling function $\pi: \ifs \rightarrow U$ that is a feasible policy for $\ifst$.
\end{problem}

We can further extend the planning problem to consider an unspecified internal system which refers to finding a \ac{DITS}, and a policy such that the resulting histories satisfy the task description, that is, he problem of jointly finding an \ac{Ispace}-filter and a policy defined over its states. 
Recall from \eqref{eq:augItS} that a policy constrains the transitions of an \ac{ITS} to the ones that are realizable under that policy. 

\begin{problem}[Find a DITS and a feasible policy]\label{prob:planning_noDITS}
Given 
$(\ifshist, U \times Y, \phi_{hist})$, for
which $\imapb:\ifshist \rightarrow \ifstask$ is a task labeling and $\ifstask$ is the corresponding \ac{Ispace}, find a sufficient \ac{Imap} $\imap:\ifshist \rightarrow \ifs$ and a feasible policy $\pi$ for $\ifst$ as a labeling function for the resulting quotient system by $\imap$. 
\end{problem}

Note that $\imapb$ can already be sufficient, so that it is the minimal sufficient refinement of itself, however, this does not necessarily imply the existence of a feasible policy defined over $\ifstask$. 
Therefore, we are not looking for a refinement of $\imapb$ while describing the \ac{DITS} over which the policy is defined.
On the other hand, we can still talk about a notion of minimality. Let $(\ifs, Y, \phi, \pi, U)$ be a state-relabeled \ac{DITS} that solves Problem~\ref{prob:planning_noDITS} (or similarly Problem~\ref{prob:planning_wDITS}) which is the quotient system of history \ac{ITS} by $\imap$, then, $(\ifs, Y, \phi, \pi, U)$ is \emph{minimal for $\pi$} if there does not exist a sufficient \ac{Imap} $\imap'$ with $\imap \succ \imap'$ for which there exists a $\pi'$ for the quotient system by $\imap'$ that satisfies $\pi(\is)=\pi'(\imap'(\is))$.

\subsubsection{Learning a sufficient ITS}

Although learning and planning overlap significantly, some unique issues arise in pure learning (see also \cite{WeiSakLav22}). This corresponds to the case when $\ifst$ is not initially given but needs to be revealed through interactions with the external system, that is, respective action and observation histories. It is assumed that whether the sentence (or sentences) describing the task is satisfied or not can be assessed at a particular history \ac{Istate}. 

We can address both filtering and planning problems defined previously within this context, considering model-free and model-based cases. In the model-free case, the task is to compute a minimal sufficient ITS that is consistent with the actions and observations. Variations include {\em lifelong learning}, in which there is a single, `long' history I-state, or more standard learning in which the system can be restarted, resulting in multiple {\em trials}, each yielding a different history I-state. In the model-based case, partial specifications of $X$, $f$, and $h$ may be given, and unknown parameters are estimated using the history I-state(s).  Different results are generally obtained depending on what assumptions are allowed.  For example, do identical history I-states imply identical state trajectories? If not, then set-based, nondeterministic models may be assumed, or even probabilistic models based on behavior observed over many trials and assumptions on probability measure, priors, statistical independence, and so on.

\section{Illustrative Examples}\label{sec:ex}
In this section we provide some simple examples to show how the ideas presented in this paper apply to filtering and planning problems. All problems can be posed as well in a machine learning context for which $\ifst$ is not given but it is revealed through interactions between the internal and external as the input-output data.  

\begin{figure}[t!]
    \centering
    \subfigure[]{\includegraphics[width=.42\linewidth]{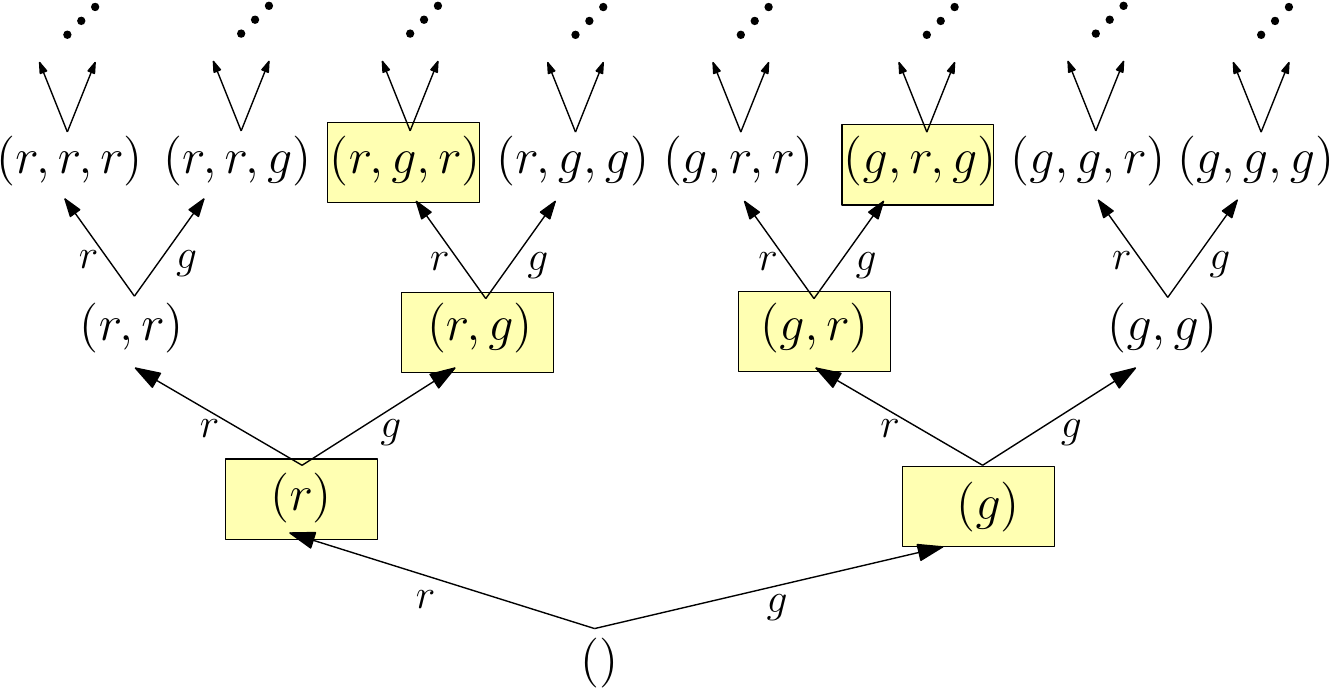}} 
	\subfigure[]{\includegraphics[width=.42\linewidth]{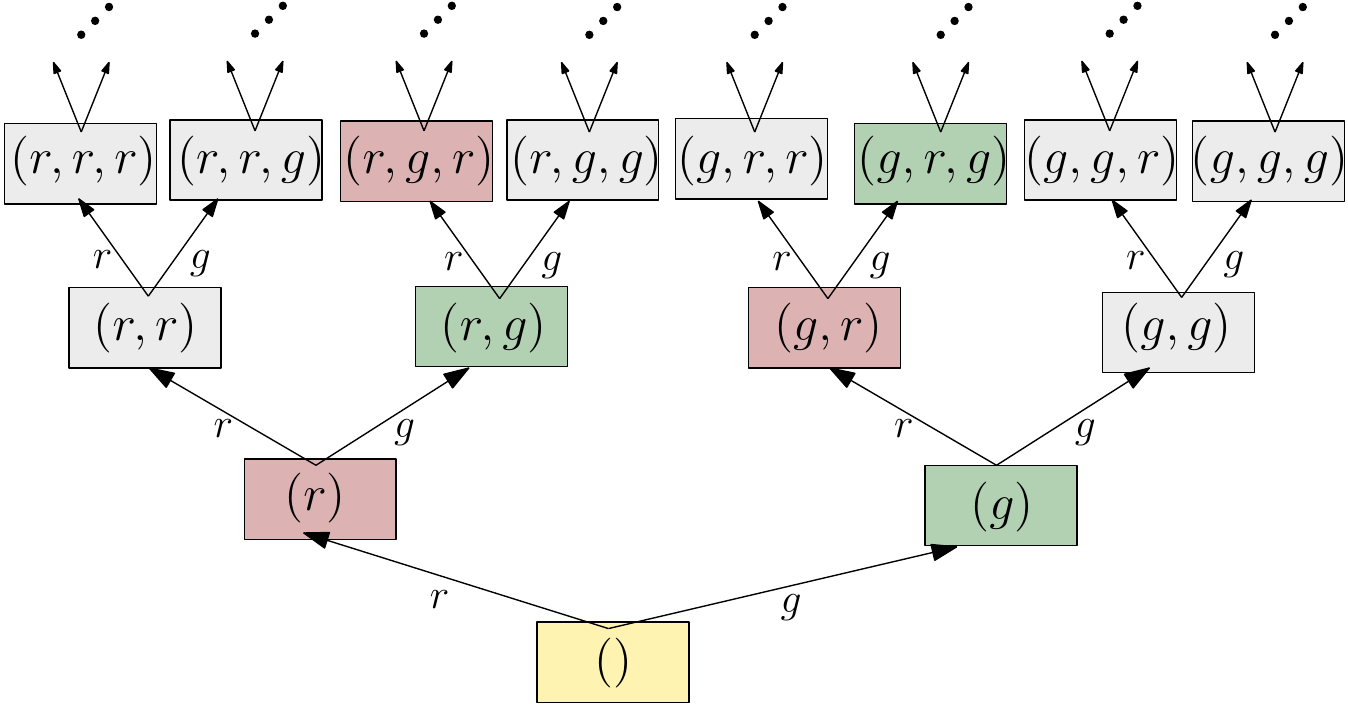}} \vspace{-3em} \\ 
	\subfigure[]{\includegraphics[width=.13\linewidth]{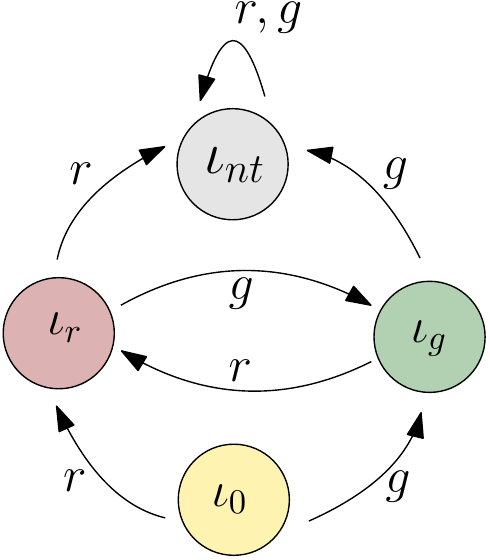}} \hspace{2em}
	\subfigure[]{\includegraphics[width=.15\linewidth]{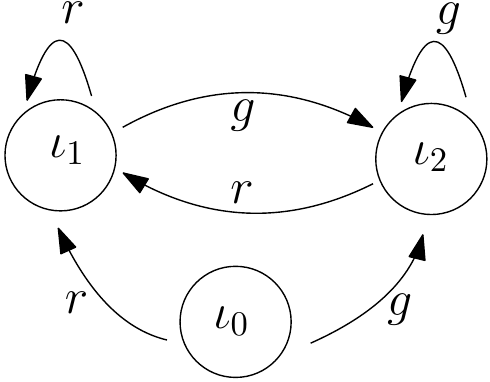}}\hspace{2em}
	\subfigure[]{\includegraphics[width=.12\linewidth]{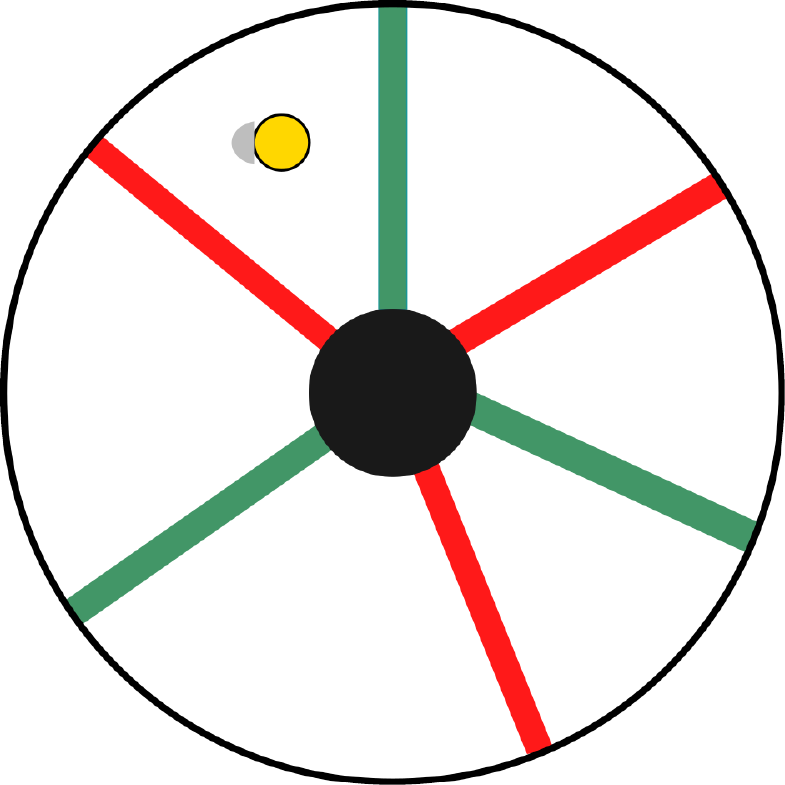}} \hspace{2em}
	\subfigure[]{\includegraphics[width=.15\linewidth]{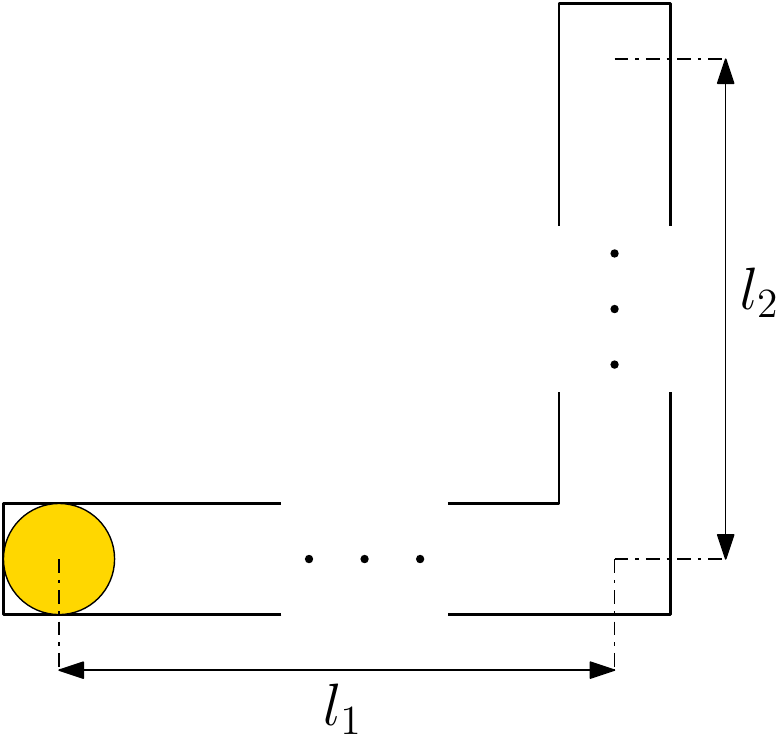}}\vspace{-1em}
    \caption{(a) State-relabeled history \ac{ITS} described in Example~\ref{ex:consistent_rot}, and labeling function $\imapb$ yellow colored states correspond to states that satisfy the task description. (b) Equivalence classes induced by $\imap'$; the minimal sufficient refinement of $\imapb$. (c) Quotient of the history \ac{ITS} by $\imap'$. (d)\ac{DITS} describing the internal system solving the planning problem described in Example~\ref{ex:plan_consistent_rot}. (e) Environment used in Examples~\ref{ex:consistent_rot},\ref{ex:plan_consistent_rot}, the obstacle (an open disk) is shown in black.  (f) L-shaped corridor; $l_1,l_2\leq l$.}
    \label{fig:red_green_gates}
    \vspace{-1.8em}
\end{figure}

Let $E\subseteq \Re^2$ be a bounded planar environment  (see Figure~\ref{fig:red_green_gates}(e)) that is partitioned into regions separated by gates 
Each gate is either green or red whose color can be detected by the robot's color sensor and follows the rule that each region shares a boundary with exactly two gates; one green and one red.
The set of possible observations are 
$Y=\{r,g\}$. 

\begin{example}\label{ex:consistent_rot} 
This example considers a filtering problem from the perspective of an independent observer. Suppose the actions taken by the robot are not observable and the only information about the system is the history of readings coming from the robot's color sensor; for example, $(r,r,r,g,r,g)$. 
Then, the history \ac{Ispace} is the set of all finite length sequences of elements of $Y$, that is, $\ifshist=Y^*$, which refers to the free monoid generated by the elements of $Y$ (or the Kleene star of $Y$). Hence, the history \ac{ITS}
can be represented as an infinite binary tree. The task is to determine whether the robot crosses the gates consistently (in a clockwise or counterclockwise manner) or not. Hence, the preimages of $\imapb: \ifshist \rightarrow \ifstask$ partition $\ifshist$ into two subsets: one which the condition is satisfied and the others. The labeling induced by $\imapb$ is shown in Figure~\ref{fig:red_green_gates}(a). Clearly, $\imapb$ is not sufficient since there exist \ac{Istate}s $\his, \his'$ such that $\imapb(\his)=\imapb(\his')$ and there exists a $y$ for which $\imapb(\fmaphist(\his,y)) \neq \imapb(\fmaphist(\his',y))$; for example consider $\his=(r,g)$, $\his'=(r,g,r)$ and $y=g$. A sufficient refinement of $\imapb$ can be obtained (equivalence classes shown in \ref{fig:red_green_gates}(b)), denote as $\imap'$, for which the quotient \ac{DITS} is shown in Figure~\ref{fig:red_green_gates}(d). Furthermore, $\imap'$ is a minimal sufficient refinement of $\imapb$ since it follows from Proposition~\ref{prob:min_suff_ref} that if a labeling is not minimal then there is a minimal one that is strictly coarser and is still sufficient. However, neither of the subsets that belong to $\ifshist/\imap'$ can be merged, since merging $\is_{nt}$ (colored gray) with anything else violates the condition that $\imap'$ is a refinement of $\imapb$ and any pairwise merge of the others violate sufficiency.
\end{example}

Suppose the robot has a boundary detector, and executes a bouncing motion using the two motion primitives $U=\{u_1,u_2\}$, in which $u_1$ is move forward and $u_2$ is rotate in place.  
Let a basic motion be move forward and bounce off the walls, which can be implemented using the elements of $U$.
We assume that the boundary detector and color sensor readings do not arrive simultaneously.

\begin{example}\label{ex:plan_consistent_rot} 
We now consider a planning problem (that belongs to the class described in Problem~\ref{prob:planning_noDITS}) for which the goal is to ensure that the robot crosses the gates consistently. 
The history \ac{Ispace} of the planner is $\ifshist=(U \times Y)^*$ and the preimages of $\imapb$ partition $\ifshist$ into two sets; the histories that satisfy the predicate and the ones that do not. 
Then, a \ac{DITS} with only three states (see Figure~\ref{fig:red_green_gates}(d)) can be derived using the mapping $\imap: \ifshist \rightarrow \ifs$, in which $\ifs=\{i_0,i_1,i_2\}$
such that for $\pi(\is_1)$ boundary with the red gate is set as a wall, $\pi(\is_2)$ the boundary with the green gate is set as a wall, and for $\pi(\is_0)$ no boundary with gates are considered as a wall.  We assume that a bouncing motion can be determined using the motion primitives so that the resulting trajectory will strike every open interval in the boundary of every region infinitely often, with non-zero, non-tangential velocities \cite{BobSanCzaGosLav11}. 
\end{example}

Consider a robot in an L-shaped planar corridor (Figure~\ref{fig:red_green_gates}(f)). Let $\mathcal{E}_l$ be the set of all such environments such that $l_1,l_2 \leq l$, in which $l_1$ and $l_2$ are the dimensions of the corridor bounded by $l$. We assume that the minimum length/width is larger than the robot radius, that is, $1$. The state space $X$ is defined as the set of all pairs $(q, E_i)$, in which $(q_1,q_2) \in E_i$, and $E_i \in \mathcal{E}_{l}$. The action set is $U=\{0,1\}\times\{0,1\}$ which corresponds to moving one step in one of the 4 cardinal directions; if the boundary is reached, the state does not change. The robot has a sensor that reports $1$ if the motion is blocked.

\begin{example}
Consider a model-based history \ac{ITS} with $\his_0 \subset X$ that specifies the initial position as $q_0=(0,0)$ but does not specify the environment so that it can be any $E_i \in \E_l$. 
Let $\ifshist$ be its set of states
and $\imap_{ndet}:\ifshist \rightarrow \pow(X)$ be an \ac{Imap} that maps 
$\his_k$ 
to a subset of $X_k \subseteq X$. 
Since $(X,U,f,h,Y)$, and $X_0=\his_0$ are known, transitions for the quotient system can be described by induction as $X_{k+1}=\hat{X}_{k+1}(X_k,u_k) \cap H(y_{k+1})$, in which $\hat{X}_{k+1}=\bigcup_{x_k\in X_k}f(x_k,u_k)$ and $H(y_{k+1})\subseteq X$ is the set of all states that could yield $y_{k+1}$. By construction, $\imap_{ndet}$ is sufficient.
Suppose $\imapb: \ifshist/\imap_{ndet} \rightarrow \ifst$ is a task labeling for localization that assigns each singleton a unique label and all the other subsets are labeled the same.
Since the transition corresponding to $([X']_{\imapb}, (u,y))$ in which $X'\subseteq X$ is not a singleton can lead to multiple labels $[x']_\imapb$, in which $x'$ is a singleton, $\imapb$ is not sufficient. Furthermore, $\imap_{ndet}$ is a minimal sufficient refinement of $\imapb$ because it is sufficient and because there does not exist a sufficient $\imap$ such that $\imap_{ndet} \succ \imap \succeq \imapb$. Suppose $\imap$ exists, that would mean some equivalence classes can be merged. However, this is not possible because merging any of the non-singleton subsets violates sufficiency (as shown for $\imapb$) and merging singletons with others violates that it is a refinement.
A policy can be described over $\ifshist/\imap_{ndet}$; $u=(1,0)$ starting from $X_0$ until $y_k=1$ is obtained and applying $u=(0,1)$ starting from $X_k$ until $y_n=1$ is obtained, then it is found that $q=(k,n)$ and $E$ is the corridor with $l_1=k$, $l_2=n$.  
\end{example}

\section{Discussion}\label{sec:con}

We have introduced a mathematical framework for determining minimal feasible policies for robot systems, by comparing ITSs over \ac{Ispace}s.
The uniqueness and minimality results are quite general: $X$ and $U$ could be discrete, typical configuration spaces, or more exotic, such as the power set of 
all functions from an infinite-dimensional Hilbert space into an infinite-dimensional Banach space.

Nevertheless, there are opportunities to expand the general theory.  For example, we assumed that $u$ is both the output of a 
policy
and the actuation stimulus in the physical world; more generally, we should introduce a mapping from an action symbol $\sigma \in \Sigma$ to a control function $\histu \in \histU$ so that plans are expressed as $\pi: \ifs \rightarrow \Sigma$ and each $\sigma=\pi(\iota)$ produces energy in the physical world via a mapping from $\Sigma$ to $\histU$. 
Another extension is to consider stochastic models, that amounts to an \ac{ITS} with a probabilistic \ac{Ispace}, and in which ways it ties to the representations used in the literature such as \ac{PSRs} \cite{LitSut01}.
Other potential extensions include continuous-time ``transitions" and novel logics to consider satisfiability questions over spaces of robot systems and tasks. An interesting direction is to consider the hardware and actuation models are variables, and fix other model components.

A grand challenge remains: The results here are only a first step toward producing a more complete and unique theory of robotics that clearly characterizes the relationships between common tasks, robot systems, environments, and algorithms that perform filtering, planning, or learning.  We should search for lattice structures that play a role similar to that of language class hierarchies in the theory of computation.  This includes the structures of the current paper and the sensor lattices of \cite{Lav12b,ZhaShe21}.  Many existing filtering, planning, and learning methods can be formally characterized within this framework, which would provide insights into relative complexity, completeness, minimality, and time/space/energy tradeoffs.  

\vspace*{-3.5mm}

\bibliographystyle{plain}
\bibliography{main,pub,more}

\end{document}